\documentclass[letterpaper]{article}

\usepackage{algorithm}
\usepackage{algorithmic}
\usepackage{amsmath}
\usepackage{amssymb}
\usepackage{amsthm}
\usepackage{bm}
\usepackage{booktabs}
\usepackage[margin=.75in]{geometry}
\usepackage{graphicx}
\usepackage{hyperref}
\usepackage{natbib}
\usepackage{subcaption}

\DeclareMathOperator*{\argmin}{arg\,min}
\newcommand*\diff{\mathop{}\!\mathrm{d}}
\newcommand*\tran{^{\mkern-1.5mu\mathsf{T}}}
\newtheorem{lemma}{Lemma}[section]
\newtheorem{theorem}{Theorem}[section]

\title{A Fast Sampling Gradient Tree Boosting Framework}
\author{Daniel Chao Zhou \and Zhongming Jin \and Tong Zhang}
\date{}

\begin{document}

\maketitle

\begin{abstract}
	As an adaptive, interpretable, robust, and accurate meta-algorithm for arbitrary differentiable loss functions, gradient tree boosting is one of the most popular machine learning techniques, though the computational expensiveness severely limits its usage. Stochastic gradient boosting could be adopted to accelerates gradient boosting by uniformly sampling training instances, but its estimator could introduce a high variance. This situation arises motivation for us to optimize gradient tree boosting. We combine gradient tree boosting with importance sampling, which achieves better performance by reducing the stochastic variance. Furthermore, we use a regularizer to improve the diagonal approximation in the Newton step of gradient boosting. The theoretical analysis supports that our strategies achieve a linear convergence rate on logistic loss. Empirical results show that our algorithm achieves a 2.5x--18x acceleration on two different gradient boosting algorithms (LogitBoost and LambdaMART) without appreciable performance loss.
\end{abstract}

\section{Introduction}

\paragraph{} Gradient tree boosting \citep{Friedman2001} is a proven technique for both classification and regression problems. Comparing to other techniques, it achieves a higher accuracy while consuming more computation resources. Figure~\ref{fig:ranking} shows that LambdaMART (a gradient tree boosting algorithm) fulfils higher accuracy but costs much longer time among several machine-learned ranking algorithms.

\begin{figure*}
	\centering
	\begin{subfigure}[t]{0.49\textwidth}
		\includegraphics[width=\textwidth]{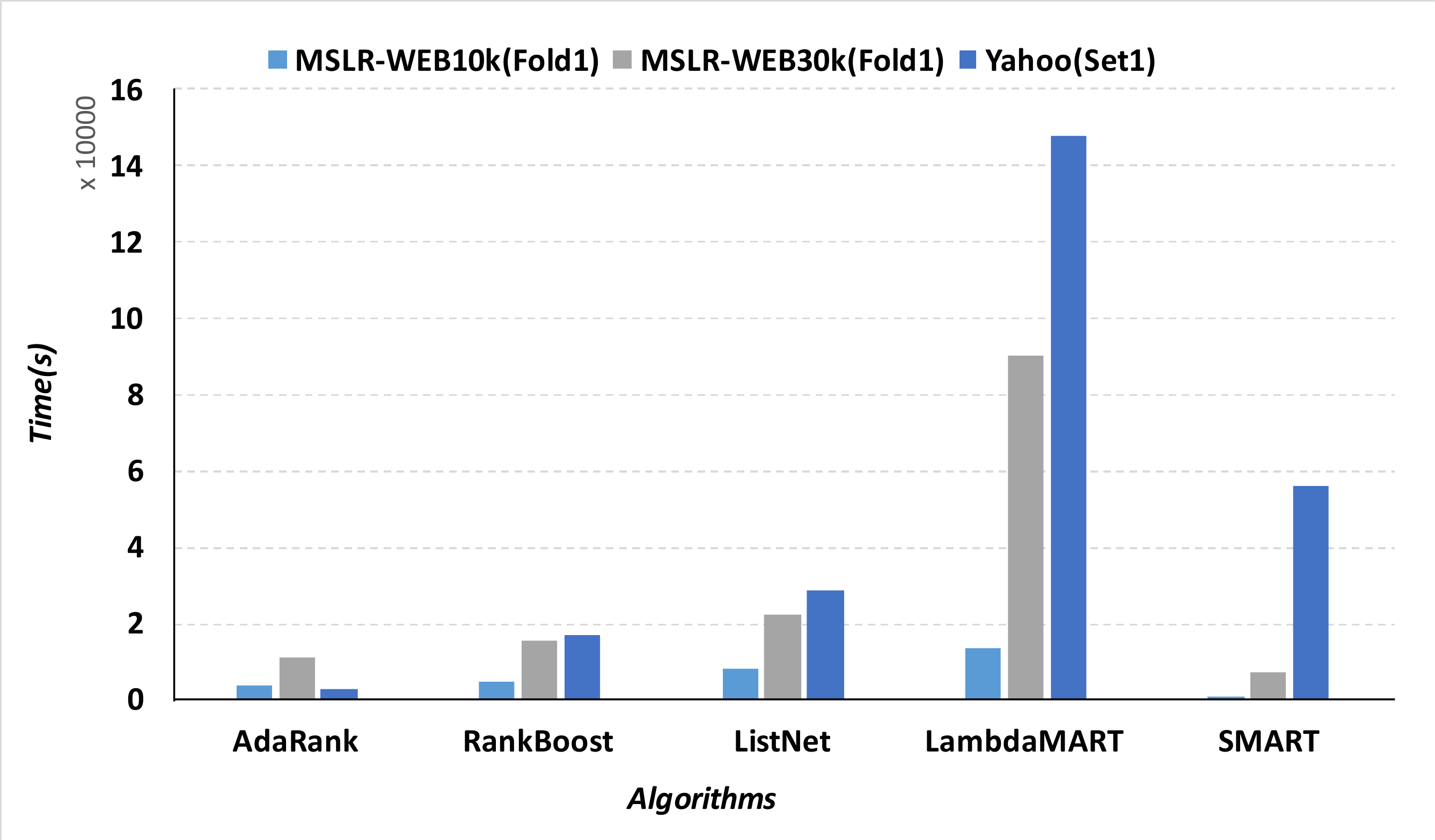}
		\caption{Efficiency}
	\end{subfigure}
	\begin{subfigure}[t]{0.49\textwidth}
		\includegraphics[width=\textwidth]{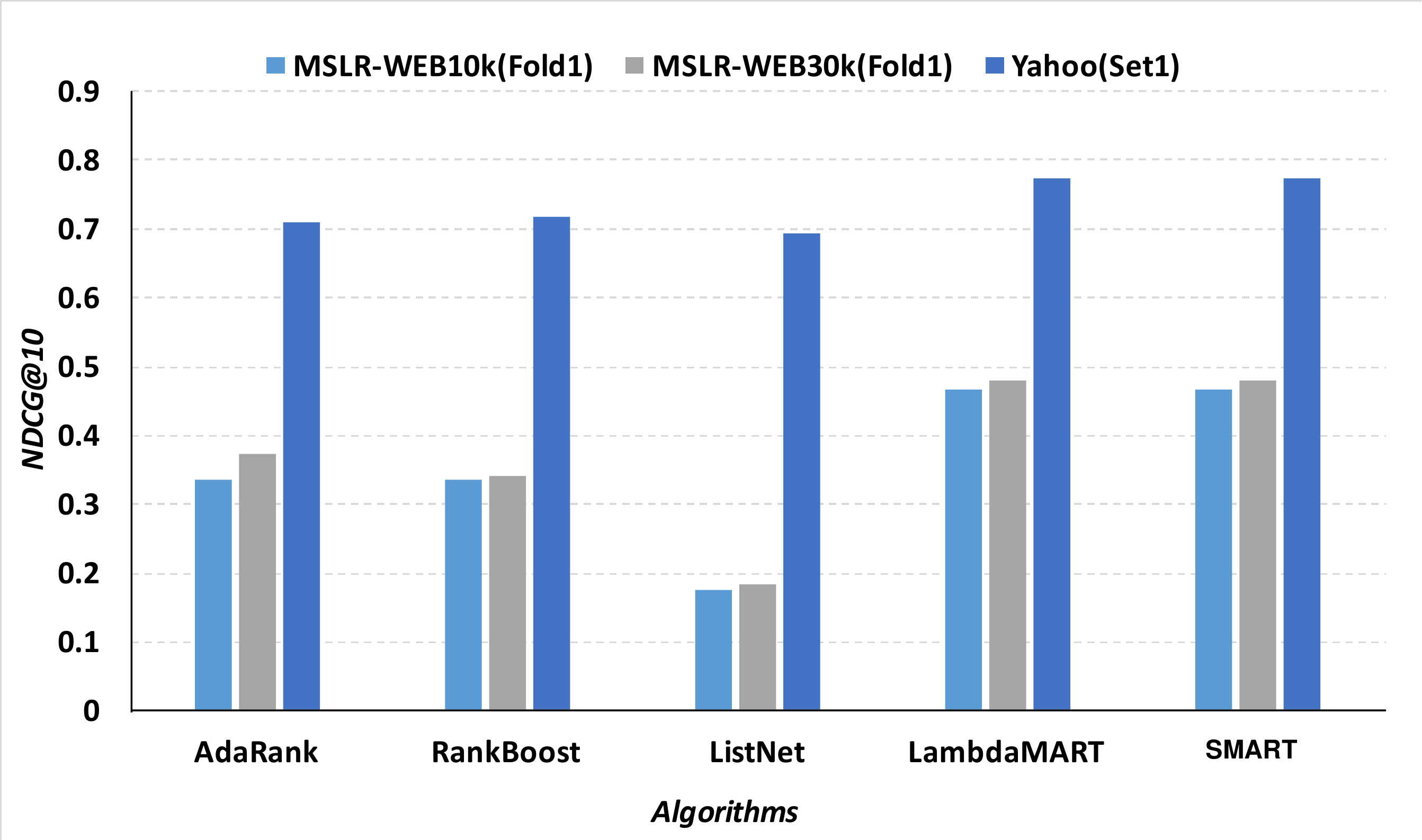}
		\caption{Accuracy (NDCG@10)}
	\end{subfigure}
	\caption{Performance of different ranking algorithms. Results of AdaRank, RankBoost, and ListNet are tested on RankLib \citep{ranklib}; LambdaMART is implemented by us according to \citep{Burges2010}.}\label{fig:ranking}
\end{figure*}

\paragraph{} Previous work such as \citep{Appel2013, Dubout2014, Johnson2014, Kalal2008, Needell2014} has been done to accelerate gradient boosting. In this paper, we propose a sampling framework named SMART (Sampling Multiple Additive Regression Tree) that accelerates gradient tree boosting through the following two methods.

\paragraph{Importance sampling} Stochastic gradient boosting \citep{Friedman2002} was extensively utilized to accelerate gradient boosting algorithms. This technique uniformly samples training instances, which guarantees the sampled dataset being an unbiased estimation of the original dataset. However, the variance of the stochastic gradient estimator could be high since the stochastic gradient varies significantly over different instances. As \citep{Zhao2015} demonstrated, the variance is minimized when sampling probability is roughly proportional to the norm of the gradient. In SMART, sampling probability is roughly proportional to the first/second-order gradient of training instances, which leads to satisfactory performance in our experiments.

\paragraph{Improved diagonal approximation} XGBoost \citep{Chen2016} uses the diagonal approximation on the Hessian in its Newton step, where non-diagonal elements in the Hessian matrix are abandoned. Such approximation has a remarkable negative impact on the algorithm. However, veritably computing the whole Hessian matrix and its inverse is computationally unfeasible. Inspired by \citep{Shamir2014}, we put a regularizer on the loss function, so the non-diagonal elements in the Hessian matrix are filled in an approximate manner.

\paragraph{} Details of the algorithms are listed in \S\ref{sec:algorithms}. To guarantee the effectiveness of the algorithms, theoretical and experimental support are given in \S\ref{sec:analysis} and \S\ref{sec:experiments}, respectively. We summarize our contribution as follows:

\begin{enumerate}
	\item We demonstrate a universal framework SMART, which is driven by \emph{importance sampling} and \emph{improved diagonal approximation}, that accelerates all gradient tree boosting algorithms;
	\item We proved that gradient tree boosting combined with importance sampling and logistic loss achieves a linear convergence rate;
	\item We implemented SMART upon LogitBoost and LambdaMART\@; our implementation achieves ideal efficiency and accuracy (2.5x--18x acceleration without appreciable accuracy loss) on multiple real-world datasets, which is superior to stochastic gradient boosting and weight trimming.
\end{enumerate}

\section{Related Work}

\subsection{Gradient Tree Boosting}

\paragraph{} Gradient tree boosting \citep{Friedman2001} produces an ensemble of regression trees as the prediction model. XGBoost \citep{Chen2016} improved the algorithm by introducing Newton's method in function space. It optimizes arbitrary second-order differentiable loss function by fit the regression tree to Newton's step. Details are listed in algorithm~\ref{alg:boosting}.

\begin{algorithm}
	\caption{Gradient tree boosting}\label{alg:boosting}
	\begin{algorithmic}
		\STATE{} $\bm y\gets\bm0$
		\FORALL{iteration $t$}
			\FORALL{instance $i$}
				\STATE $g_i\gets\partial\ell/\partial\hat y_i$
				\STATE $h_i\gets\partial^2\ell/\partial\hat y_i^2$
			\ENDFOR{}
			\STATE fit $f$ to the Newton's step $-\bm g/\bm h$ with diagonal approximation
			\STATE{} $\bm y\gets\bm y+\nu f(\bm X)$
		\ENDFOR{}
		\RETURN{} all decision trees
	\end{algorithmic}
\end{algorithm}




\subsection{LogitBoost}

\paragraph{} LogitBoost \citep{Friedman2000} solves the binary classification problem by applying the logistic loss function to gradient boosting, where the loss function is

$$\ell(\hat y_i)=y_i\ln\frac{1}{\psi_i}+(1-y_i)\ln\frac{1}{1-\psi_i},$$

and the binary label $y_i\in \{0,1\}$, $\psi_i$ is the probability of label being positive

$$\psi_i=\Pr(y_i=1)=\frac{e^{\hat y_i}}{e^{\hat y_i}+e^{-\hat y_i}}.$$

LogitBoost is the straightforward combination of algorithm~\ref{alg:boosting} and the first \& second gradient of the logistic loss

$$
\bm g=\nabla\ell(\bm{\hat y})=2(\bm\psi-\bm y),
$$

and

$$\bm h=\nabla^2\ell(\bm{\hat y})=4\langle\bm\psi,1-\bm\psi\rangle.
$$

\subsection{LambdaMART}

\paragraph{} Learning to rank is a task to train a ranking model with machine learning techniques, such that the model sorts new instances according to their relevance, preference, or importance. Previous work has shown that LambdaMART \citep{Burges2010} is an outperforming solution to the ranking problem, which championed the Yahoo! Learning to Rank Challenge \citep{Chapelle2011}. LambdaMART is the algorithm combining gradient tree boosting and the following loss function

$$
\ell(\bm s)=\sum_{\{i,j\}\in I}|\Delta\text{NDCG}_{ij}|\log\left(1+e^{-\sigma(s_i-s_j)}\right),
$$

where $\bm s$ is the model output, $I$ is the set of preference pairs from the training data, and NDCG is an indicator measuring the quality of a permutation. From the above loss function we have

\begin{align*}
	&g_i=\frac{\partial\ell(\bm s)}{\partial s_i}\\
&=
\sum_{j:\{i,j\}\in I}\frac{-\sigma|\Delta\text{NDCG}_{ij}|}{1+e^{\sigma(s_i-s_j)}}
+
\sum_{j:\{j,i\}\in I}\frac{\sigma|\Delta\text{NDCG}_{ij}|}{1+e^{\sigma(s_j-s_i)}},
\end{align*}

and

\begin{align*}
	&h_i=\frac{\partial^2\ell(\bm s)}{\partial s_i^2}
\\&=
\sum_{j:\{i,j\}\in I}
\frac{\sigma^2|\Delta\text{NDCG}_{ij}|e^{\sigma(s_i-s_j)}}{\left(1+e^{\sigma(s_i-s_j)}\right)^2}
\\
&+
\sum_{j:\{j,i\}\in I}\frac{\sigma^2|\Delta\text{NDCG}_{ij}|e^{\sigma(s_j-s_i)}}
{\left(1+e^{\sigma(s_j-s_i)}\right)^2}.
\end{align*}

LambdaMART is the straightforward combination of algorithm~\ref{alg:boosting} and the above $g_i$, $h_i$.

\subsection{Weight trimming}

\paragraph{} First introduced in \citep{Friedman2000}, weight trimming is an idea that correctly classified instances with high confidence, viz.\ instances with \emph{weight} less than a certain threshold, should be abandoned in each boosting iteration. Friedman observed that the computation time was dramatically reduced while accuracy was not sacrificed.


\subsection{Stochastic Gradient Boosting}

\paragraph{} \citep{Friedman2002} proposed that a tree should be fit on a random subsample of the training set in each boosting iteration. Instead of abandoning instances with small weights, stochastic gradient boosting samples training instances uniformly, v.i.z.\ that whether each instance being sampled are i.i.d.\ Bernoulli variables. Friedman observed that the accuracy of gradient boosting was substantially improved.



\section{Algorithms}
\label{sec:algorithms}

\paragraph{} SMART consists of two strategies, both of which accelerates arbitrary gradient tree boosting algorithms with a second-order differentiable loss function. The first strategy is the gradient boosting with Newton's method combined with importance sampling, where the importance is defined as the first gradient. The second strategy is the gradient boosting with Newton's method combined with importance sampling and improved diagonal approximation, where the importance is defined as the second gradient.

\subsection{First-gradient-proportional sampling}

\paragraph{} In the regression tree construction procedure, all instances are enumerated several times in each boosting step. It is reasonable to use a subsample of instances to construct each regression tree. In this algorithm, we introduce importance sampling into gradient boosting, which is superior to uniform sampling by achieving lower variance. In the same way as \citep{Friedman2002} did, we make independent decision to sample each instance on each iteration. Furthermore, we define Bernoulli variable $q_i^{(t)}\in \{0,1\}$ indicating whether $i^\text{th}$ instance is utilized on $t^\text{th}$ iteration, and $p_i^{(t)}=\Pr\left(q_i^{(t)}=1\right)$ being the sampling probability.

\paragraph{} Differing from \citep{Friedman2002}, $\{q_i\}_1^N$ are not i.i.d.\ Bernoulli variables. As \citep{Zhao2015} demonstrated, the variance of the gradient estimator is minimized when the sampling probability is roughly proportional to the norm of the gradient, such that $$p_i=\min(1,\rho|g_i|),$$ where $\rho$ is a tuning parameter.

\paragraph{} In order to eliminate the bias of the empirical risk introduced by the sampling algorithm, the regression tree is constructed upon the importance-balanced loss function of

$$\argmin_f\sum_{i=0}^n\frac{q_i}{p_i}\ell(f(\bm x_i)).$$

\paragraph{} First and second gradient is derived from this loss function such that $\tilde g_i=\frac{q_i}{p_i}g_i$ and $\tilde h_i=\frac{q_i}{p_i}h_i$. This algorithm is listed in algorithm~\ref{alg:boosting_i}.

\begin{algorithm}
	\caption{First-gradient-proportional sampling}\label{alg:boosting_i}
	\begin{algorithmic}
		\STATE $\bm y\gets\bm0$
		\FORALL{iteration $t$}
			\FORALL{instance $i\in \{1\dots n\}$}
				\STATE{$g_i\gets\partial\ell/\partial y_i$}
				\STATE{$h_i\gets\partial^2\ell/\partial y_i^2$}
				\STATE $p_i\gets\min(1,\rho|g_i|)$
				\STATE{generate value $q_i\in\{0,1\}$ by Bernoulli parameter $p_i$}
				\STATE $g_i\gets\frac{q_i}{p_i}g_i$
				\STATE $h_i\gets\frac{q_i}{p_i}h_i$
			\ENDFOR{}
			\STATE fit $f$ to the Newton's step $f(\bm x_i)=-g_i/h_i$ from the subset $\mathcal S=\{i|q_i=1\}$
			\STATE $\bm y\gets\bm y+\nu f(\bm X)$
		\ENDFOR{}
		\RETURN all decision trees
	\end{algorithmic}
\end{algorithm}




\subsection{Second-gradient-proportional sampling}

\paragraph{} XGBoost \citep{Chen2016} introduced Newton's method into gradient boosting, which accelerates the algorithm dramatically. However, non-diagonal elements in the Hessian matrix are abandoned in Newton's step. Such approximation gives a negative impact on the algorithm. \citep{Shamir2014} demonstrated that by solving the problem


\begin{equation}
	\label{eq:shamir}
	\argmin_f\sum_{i=1}^n
	\left(
		\ell(\tilde y_i)
		+\langle\ell'(\tilde y),f(\bm x_i)-\tilde y_i\rangle
		+\frac{1}{\eta}D(f(\bm x_i),\tilde y_i)
	\right)
\end{equation}

with Newton's method, we get a more accurate Newton's step, although the Hessian matrix is not entirely computed. In equation~\ref{eq:shamir}, the $\tilde y_i$ is a rough estimation of $y_i$. Specifically, it is the the hypothesis of last two iterations in our implementation. $\ell'(\tilde y)$ is defined as

$$
\ell'(\tilde y)=\frac{1}{|I_j|}\sum_{\imath\in I_j}\ell'(\tilde y_\imath),
$$

where $I_j$ is the instance set of the regression tree leaf $j$ which covers instance $i$. $D(\cdot,\cdot)$ is the Bregman divergence corresponding to the loss function $\ell(\cdot)$

$$
D(f(\bm x_i),\tilde y_i)=\ell(f(\bm x_i))-\ell(\tilde y_i)-\langle\ell'(\tilde y_i),f(\bm x_i)-\tilde y_i\rangle.
$$

Plugging $\ell'(\tilde y)$ and $D(f(\bm x_i),\tilde y_i)$ into~\ref{eq:shamir}, eliminating constant items, we have

$$
\argmin_f\left(
	\ell(f(\bm x_i))-\ell'(\tilde y_i)f(\bm x_i)
	+\frac{\eta f(\bm x_i)}{|I_j|}\sum_{\imath\in I_j}\ell'(\tilde y_\imath)
\right).
$$

Applying importance balancing to eliminate the bias, we have the final loss function

$$
L(f(\bm X))=\frac{q_i}{p_i}
\left(
	\ell(f(\bm x_i))-\ell'(\tilde y_i)f(\bm x_i)
	+\frac{\eta f(\bm x_i)}{|I_j|}\sum_{\imath\in I_j}\ell'(\tilde y_\imath)
\right).
$$

From this loss function, we have the first and second gradient for Newton's step

$$
\tilde g_i=\frac{q_i}{p_i}\left(
	g_i-\ell'(\tilde y_i)+\frac{\eta}{|I_j|}\sum_{\imath\in I_j}\ell'(\tilde y_\imath)
\right),
$$

and

$$
\tilde h_i=\frac{q_i}{p_i}h_i.
$$

In this algorithm, the sampling probability is roughly proportional to the second gradient such that $p_i=\min(1,\rho h_i)$.

\begin{algorithm}
	\caption{Second-gradient-proportional sampling}\label{alg:boosting_ii}
	\begin{algorithmic}
		\STATE $\bm y\gets\bm0$
		\FORALL{iteration $t$}
			\FORALL{instance $i\in \{1\dots n\}$}
				\STATE{$g_i\gets\partial\ell/\partial y_i$}
				\STATE{$h_i\gets\partial^2\ell/\partial y_i^2$}
				\STATE $p_i\gets\min(1,\rho h_i)$
				\STATE{generate value $q_i\in\{0,1\}$ by Bernoulli parameter $p_i$}
				\STATE $g_i\gets\frac{q_i}{p_i}\left(g_i-\ell'(\tilde y_i)+\eta\ell'(\tilde y)\right)$
				\STATE $h_i\gets\frac{q_i}{p_i}h_i$
			\ENDFOR{}
			\STATE fit $f$ to the Newton's step $f(\bm x_i)=-g_i/h_i$ from the subset $\mathcal S=\{i|q_i=1\}$
			\STATE $\bm y\gets\bm y+\nu f(\bm X)$
		\ENDFOR{}
		\RETURN all decision trees
	\end{algorithmic}
\end{algorithm}

\section{Analysis}
\label{sec:analysis}

\paragraph{} In this section, we provide a theoretical guarantee that our first-gradient-proportional subsampling achieves a linear convergence rate over the logistic loss. Before presenting the result, we give some definitions as follows.

\begin{enumerate}
	\item Let $\bm V\in\mathbb R^{N\times J}$ encoding the tree structure which transforms instances to leaf nodes, such that $v_{ij}=1$ if the $i^\text{th}$ instance sinks into the $j^\text{th}$ leaf, and $0$ otherwise.
	\item Let $\bm Q=\mathrm{diag}\left(\begin{smallmatrix}\frac{q_1}{p_1}&\cdots&\frac{q_N}{p_N}\end{smallmatrix}\right)$.
\end{enumerate}

Under such definitions, Newton's step of the legacy gradient tree boosting is

$$\bm f=-\nu\bm V(\bm V\tran\bm H\bm V)^{-1}\bm V\tran\bm g,$$

and of the first-gradient-proportional subsampling algorithm is

$$\bm f=-\nu\bm V(\bm V\tran\bm H\bm Q\bm V)^{-1}\bm V\tran\bm Q\bm g.$$



such that $\bm y^+=\bm y+\bm f$, where $\bm y^+$ is the prediction after another boosting iteration. The linear convergence rate is established upon the following two theorems.

\begin{theorem}
	\label{the:first0}
	Given a sequence $\epsilon_0>\epsilon_1>\cdots>\epsilon_T>0$ with the recurrence relation $\epsilon_{t+1}\le\gamma\epsilon_t$ where $\gamma\in(0,1)$ is a small constant, we have a linear convergence that
	$$\epsilon_T<\epsilon_0\gamma^T.$$
\end{theorem}

\begin{theorem}
	\label{the:first1}
	Given that
	$$\bm f=-\nu\bm V(\bm V\tran\bm H\bm Q\bm V)^{-1}\bm V\tran\bm Q\bm g,$$
	we have
	$$L\left(\bm y^+\right)=L(\bm y+\bm f)\le\gamma L(\bm y),$$
	where $0<\gamma<1$.
\end{theorem}

Theorem~\ref{the:first0} is self-evident and its proof is omitted here. The proof for theorem~\ref{the:first1} is given in appendix~\ref{app:proof_first1}.

\section{Experiments}
\label{sec:experiments}

\paragraph{} We implement SMART upon two gradient tree boosting algorithms, LogitBoost and LambdaMART, each with two subsampling strategies. Multiple real-world datasets are employed to verify the efficiency of SMART.

\paragraph{} LogitBoost is tested on the LIBSVM/a8a datasets \citep{Chang2011}; LambdaMART is tested on three large-scale datasets: MSLR-WEB10K/MSLR-WEB30K of LETOR datasets \citep{Qin2013}, and Yahoo! Learning to Rank Challenge dataset \citep{Chapelle2011}. Principal characteristics of these datasets are given in table~\ref{tab:datasets}. Note that a8a is a binary classification dataset, while the others are ranking datasets.

\begin{table}
  \centering
  \caption{Properties of Datasets}
  \label{tab:datasets}
  \begin{tabular}{lrrrr}
    \toprule
    dataset & \#train & \#test & \#dim\\
    \midrule
    LIBSVM/a8a & 22696 & 9865 & 123\\
    MSLR-WEB10K/Fold1 & 719311 &  241521 & 136\\
    MSLR-WEB30K/Fold1 & 2258066 & 753611 & 136\\
    Yahoo/Set1 & 466687 & 165660 & 700\\
    \bottomrule
  \end{tabular}
\end{table}

\paragraph{} SMART is compared with 1) legacy LogitBoost or LambdaMART without any instance selecting strategies as the baseline, 2) weight-trimming, and 3) stochastic gradient boosting (uniform subsampling).

\paragraph{} Results are shown in figure~\ref{fig:a8a},~\ref{fig:mslr10k},~\ref{fig:mslr30k}, and~\ref{fig:yahoo}, from which we can tell that 1) second-gradient-proportional strategy utilizes the least training instances when it converges to its ideal accuracy; and 2) as boosting iteration proceeds, second-gradient-proportional strategy achieves much better accuracy than (or almost same accuracy with) the baseline.

\paragraph{} Each of table~\ref{tab:a8a},~\ref{tab:mslr10k},~\ref{tab:mslr30k},~\ref{tab:yahoo} consists of two parts. The upper part shows the performance details when each algorithm achieves a relatively ideal accuracy; the lower part shows the details when each algorithm achieves their best accuracy. Take MSLR-WEB10K for example, when it achieves an NDCG@10 at 0.453\footnote{The best NDCG@10 of MSLR-WEB10K achieved by RankLib is 0.4478.}, the second-gradient-proportional subsampling strategy takes 54 iterations and averagely utilizes 6\% training instances in each iteration; when it achieves an NDCG@10 at 0.465, the second-gradient-proportional subsampling strategy only takes 730 seconds, which achieves an 18x acceleration.

\paragraph{} In conclusion, our second-gradient-proportional subsampling strategy precisely portrays the characteristics of the original dataset while consumes the least sample size. Taking efficiency and accuracy into account, the second-gradient-proportional subsampling strategy is superior to any other algorithms in our experiments.

\begin{figure*}[ht]
  \begin{minipage}[b]{.5\linewidth}
    \includegraphics[width=\linewidth]{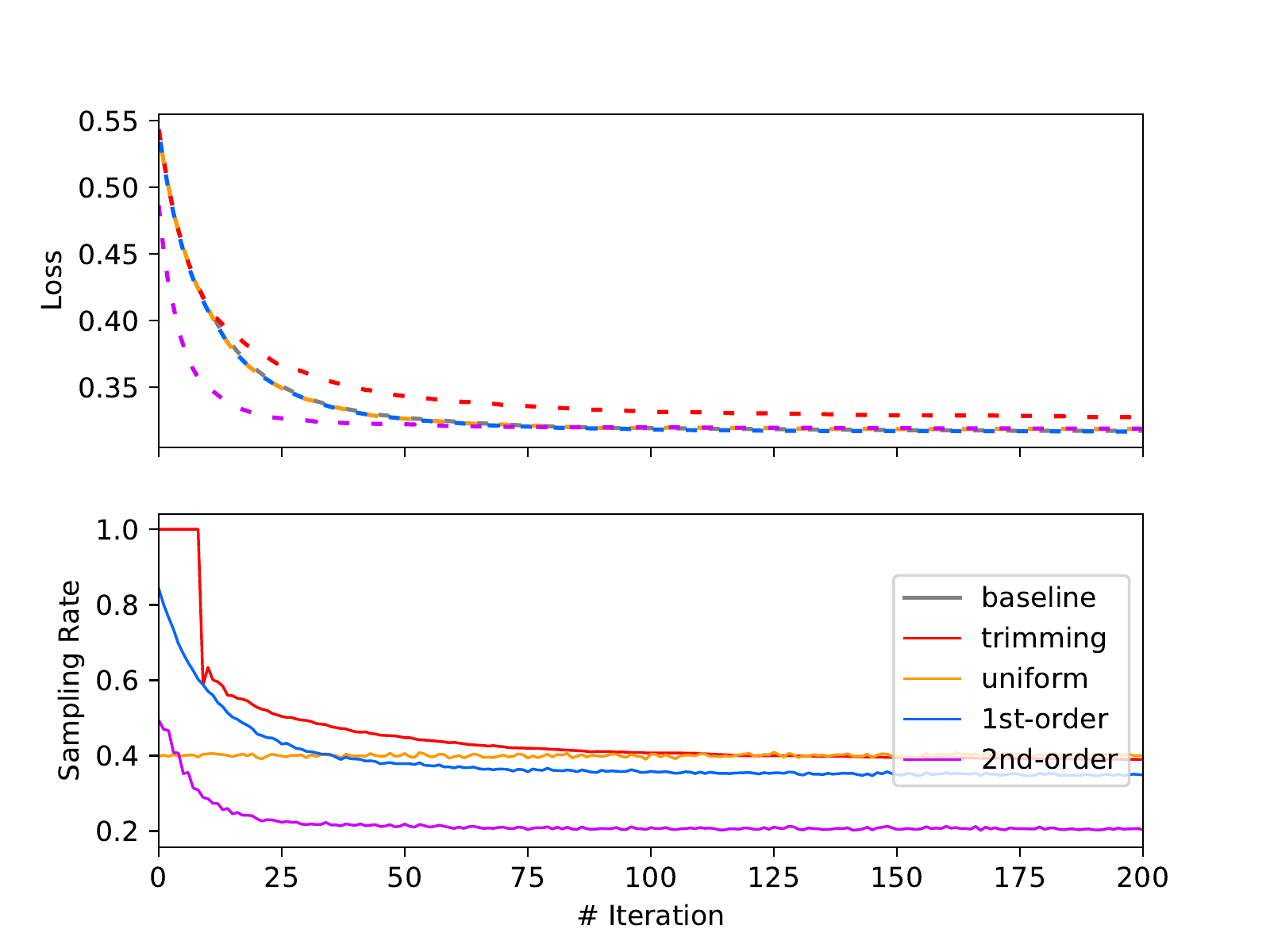}
    \captionof{figure}{Loss curve on a8a}
    \label{fig:a8a}
  \end{minipage}
  \begin{minipage}[b]{.5\linewidth}
    \centering
    \begin{tabular}{lrrrr}
      \toprule
      & loss & iter & ASRI\footnotemark & time/s \\
      \midrule
      baseline  & 0.325 &  60 & 1.0000 & 65 \\
      uniform   & 0.325 &  58 & 0.3998 & 16 \\
      trimming  & 0.325 & 443 & 0.4195 & 141 \\
      1st-order & 0.325 &  57 & 0.5028 & 15 \\
      2nd-order & 0.325 &  34 & 0.3052 & 6 \\
      \\
      baseline  & 0.319 &  97 & 1.0000 & 127\\
      uniform   & 0.319 &  96 & 0.3993 & 24 \\
      trimming  & 0.325 & 443 & 0.4195 & 141 \\
      1st-order & 0.319 &  89 & 0.4531 & 22 \\
      2nd-order & 0.319 & 142 & 0.2317 & 20 \\
      \bottomrule
    \end{tabular}
    \captionof{table}{Performance on a8a}
    \label{tab:a8a}
  \end{minipage}
\end{figure*}

\begin{figure*}[p]
  \begin{minipage}[b]{.5\linewidth}
    \includegraphics[width=\linewidth]{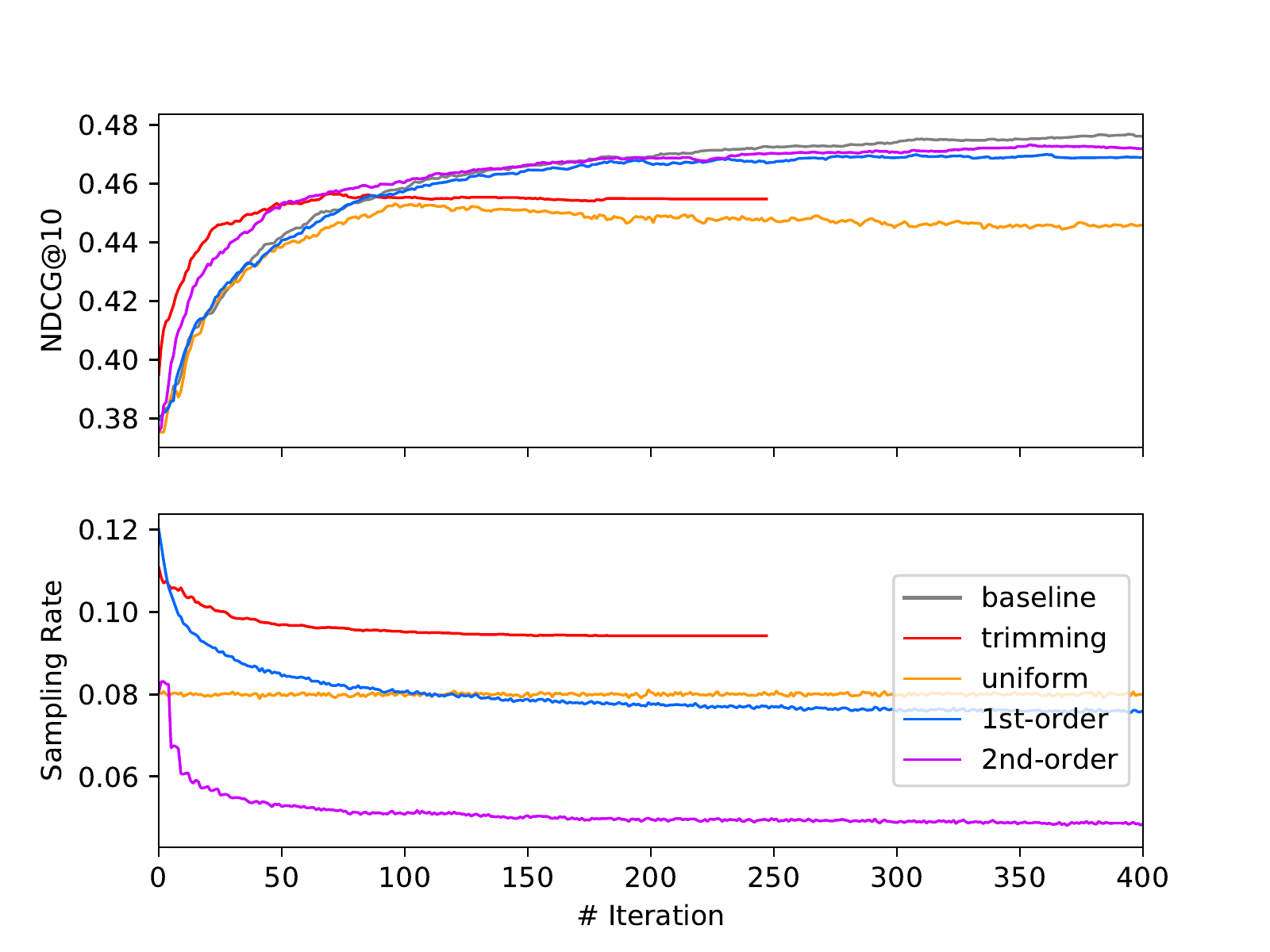}
    \caption{NDCG curve on MSLR-WEB10K}
    \label{fig:mslr10k}
  \end{minipage}
  \begin{minipage}[b]{.5\linewidth}
    \centering
    \begin{tabular}{lrrrr}
      \toprule
      & NDCG & iter & ASRI & time/s \\
      \midrule
      baseline  & 0.453 &  83 & 1.0000 &  6899 \\
      uniform   & 0.453 & 100 & 0.0799 &   779 \\
      trimming  & 0.453 &  52 & 0.1018 &   403 \\
      1st-order & 0.453 &  83 & 0.0902 &   665 \\
      2nd-order & 0.453 &  54 & 0.0604 &   273 \\
      \\
      baseline  & 0.465 & 141 & 1.0000 & 13190 \\
      uniform   & 0.453 & 100 & 0.0799 &   779 \\
      trimming  & 0.456 &  72 & 0.1004 &   571 \\
      1st-order & 0.465 & 162 & 0.0851 &  1349 \\
      2nd-order & 0.465 & 138 & 0.0549 &   730 \\
      \bottomrule
    \end{tabular}
    \caption{Performance on MSLR-WEB10K}
    \label{tab:mslr10k}
  \end{minipage}
\end{figure*}

\begin{figure*}[p]
  \begin{minipage}[b]{.5\linewidth}
    \includegraphics[width=\linewidth]{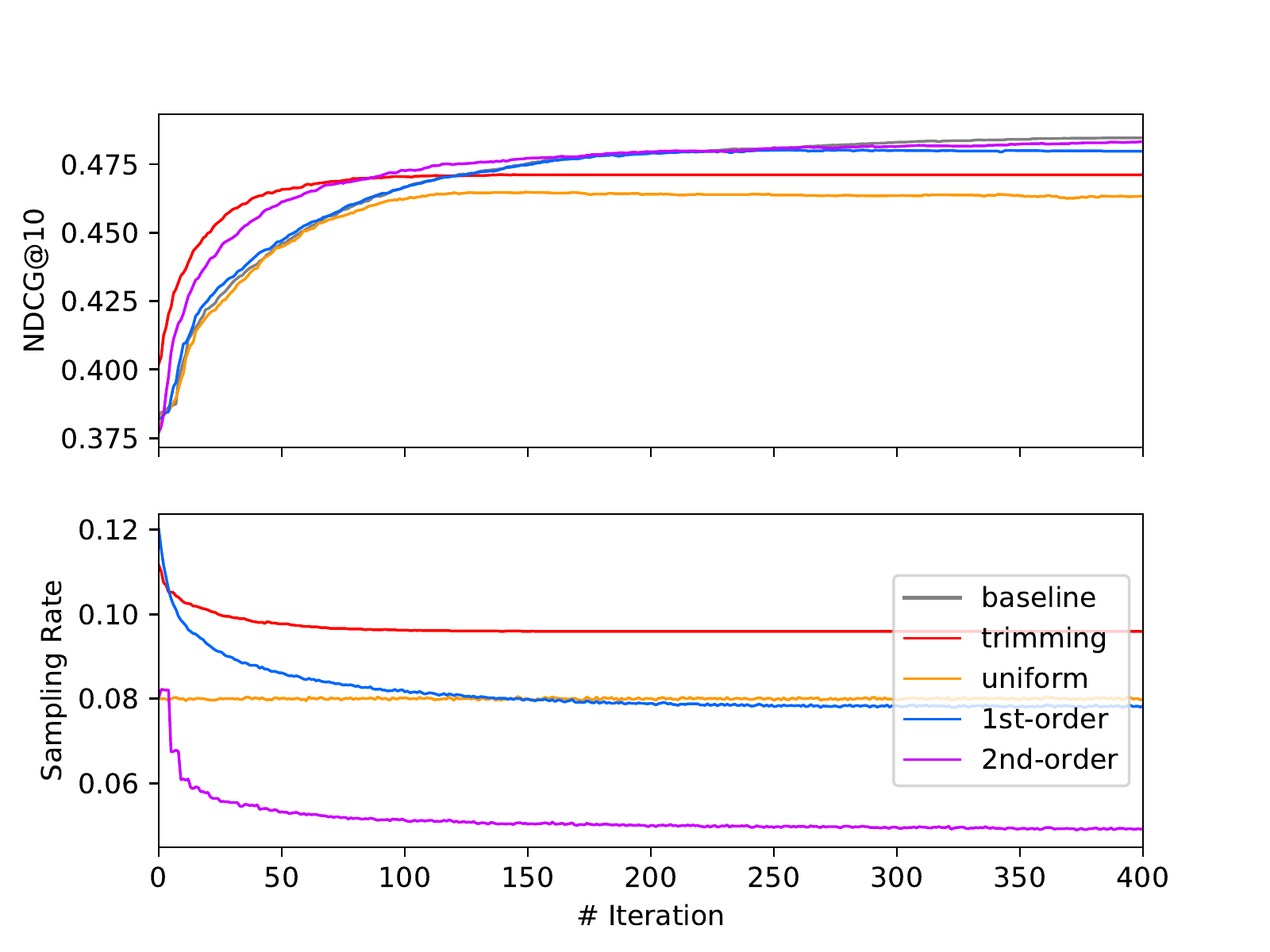}
    \caption{NDCG curve on MSLR-WEB30K}
    \label{fig:mslr30k}
  \end{minipage}
  \begin{minipage}[b]{.5\linewidth}
    \centering
    \begin{tabular}{lrrrr}
      \toprule
      & NDCG & iter & ASRI & time/s \\
      \midrule
      baseline  & 0.464 &  97 & 1.0000 & 33773 \\
      uniform   & 0.464 & 117 & 0.0800 &  6495 \\
      trimming  & 0.464 &  47 & 0.1021 &  2948 \\
      loss-based & 0.464 & 95 & 0.1468 &  7453 \\
      1st-order & 0.464 &  94 & 0.0901 &  5652 \\
      2nd-order & 0.464 &  63 & 0.0595 &  1654 \\
      \\
      baseline  & 0.480 & 228 & 1.0000 & 89898 \\
      uniform   & 0.464 & 117 & 0.0800 &  6495 \\
      trimming  & 0.471 & 139 & 0.0984 &  9405 \\
      1st-order & 0.480 & 247 & 0.0836 & 16234 \\
      2nd-order & 0.480 & 236 & 0.0530 &  6827 \\
      \bottomrule
    \end{tabular}
    \caption{Performance on MSLR-WEB30K}
    \label{tab:mslr30k}
  \end{minipage}
\end{figure*}

\begin{figure*}[p]
  \begin{minipage}[b]{.5\linewidth}
    \includegraphics[width=\linewidth]{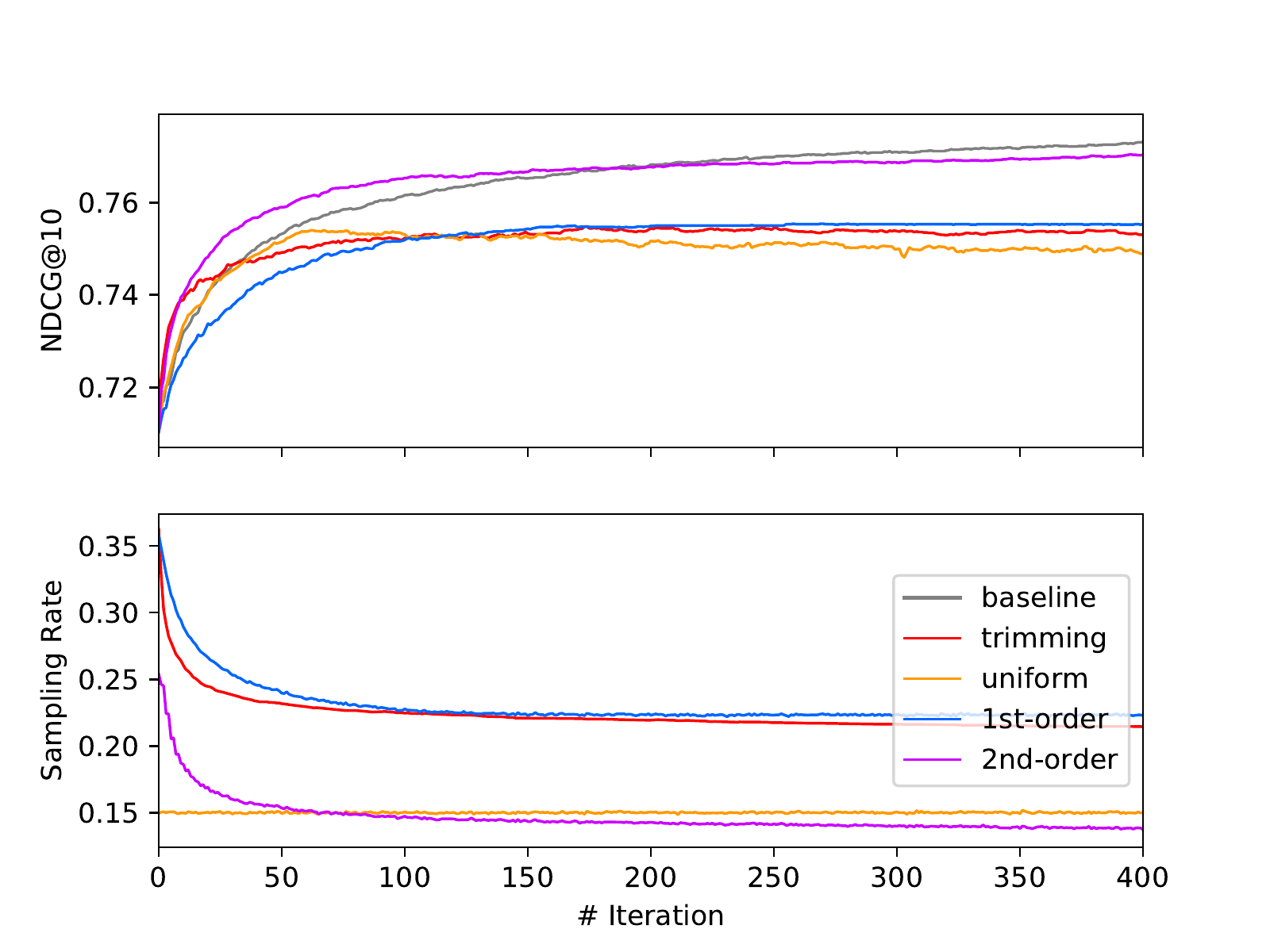}
    \caption{NDCG curve on Yahoo! dataset}
    \label{fig:yahoo}
  \end{minipage}
  \begin{minipage}[b]{.5\linewidth}
    \centering
    \begin{tabular}{lrrrr}
      \toprule
      & NDCG & iter & ASRI & time/s \\
      \midrule
      baseline  & 0.754 &  57 & 1.0000 & 18200 \\
      uniform   & 0.754 &  66 & 0.1500 &  2623 \\
      trimming  & 0.754 & 169 & 0.2352 & 10156 \\
      1st-order & 0.754 & 147 & 0.2455 & 11389 \\
      2nd-order & 0.754 &  35 & 0.1804 &  1769 \\
      \\
      baseline  & 0.772 & 357 & 1.0000 & 147872 \\
      uniform   & 0.754 &  66 & 0.1500 &   2623 \\
      trimming  & 0.755 & 632 & 0.2201 &  35355 \\
      1st-order & 0.755 & 173 & 0.2423 &  14118 \\
      2nd-order & 0.772 & 645 & 0.1430 &  55651 \\
      \bottomrule
    \end{tabular}
    \caption{Performance on Yahoo! dataset}
    \label{tab:yahoo}
  \end{minipage}
\end{figure*}

\footnotetext{ASRI\@: average subsampling rate per iteration}

\section{Conclusion}

\paragraph{} In this paper, we presented SMART, an efficient subsampling framework based on gradient tree boosting. Inspired by \citep{Zhao2015}, we introduced importance sampling such that subsampling probability is roughly proportional to the gradient. Inspired by \citep{Shamir2014}, we fill the Hessian matrix in an approximate manner. Theoretically, we proved that SMART achieves a linear convergence rate on logistic loss, which is same as the legacy LogitBoost. Practically, we implemented SMART upon LogitBoost and LambdaMART\@. Experiments show that our implementation achieves a 2.5x--18x acceleration under real-world, large-scale datasets.

\paragraph{} Furthermore, SMART could be applied on more gradient boosting algorithms, such as AdaBoost, for further guarantee.

\appendix

\section{Proof of theorem~\ref{the:first1}}
\label{app:proof_first1}

Following auxiliary lemmas are needed to prove theorem~\ref{the:first1}.

\begin{lemma}
  \label{lem:bound_first_0}
  $(\bm V\tran\bm H_{\bm\xi}\bm Q\bm V)(\bm V\tran\bm H\bm Q\bm V)^{-1}\le\mu\bm I$, where $\bm H_{\bm\xi}=\mathrm{diag}(\begin{smallmatrix}\ell''(\xi_1)&\dots&\ell''(\xi_N)\end{smallmatrix})$.
\end{lemma}

\begin{proof}
  \citep{Sun2014}'s lemma 9 proved that
  $$\frac{\ell''(\xi)}{\ell''(y)}\le e^{2|\xi-y|},$$
  which gives
  $$\sum_{i\in I_j}\frac{q_i\ell''(\xi_i)}{p_i}\le\sum_{i\in I_j}\frac{q_i\ell''(y_i) e^{2|\xi_i-y_i|}}{p_i},$$
  which is equivalent to
  $$(\bm V\tran\bm H_{\bm\xi}\bm Q\bm V)(\bm V\tran\bm H\bm Q\bm V)^{-1}\le\exp(\max_i|\xi_i-y_i|)\bm I.$$
\end{proof}

\begin{lemma}
  \label{lem:bound_first_1}
  $\bm g\tran\bm Q\bm V(\bm V\tran\bm H\bm Q\bm V)^{-1}\bm V\tran\bm Q\bm g\le\gamma^2\bm g\tran\bm Q(\bm H\bm Q)^{-1}\bm Q\bm g$.
\end{lemma}

\begin{proof}
  Recall that $\ell(y)=-r\ln \psi-(1-r)\ln(1-\psi)$, and $\psi=\frac{e^{y}}{e^{y}+e^{-y}}$, which yield $\frac{\diff\psi}{\diff y}=2\psi(1-\psi)$, $g=\ell'(y)=2(\psi-r)$, and $h=\ell''(y)=4\psi(1-\psi)$. Defining $\phi=|r-\psi|$, we have $g^2=4\phi^2$, $h=4\phi(1-\phi)$, and $\frac{g^2}{h}=\frac{\phi}{1-\phi}$. The $\psi_i$ clapping yields $\frac{1}{1-\phi_i}\le\frac{1}{\rho}$, and $1-\phi_i\le1$. We have
  \begin{align*}
    &\left(\sum_{i\in I_j}\frac{q_ih_i}{p_i}\right)\left(\sum_{i\in I_j}\frac{q_ig_i^2}{p_ih_i}\right)\\
    =&\left(\sum_{i\in I_j}\frac{4q_i\phi_i(1-\phi_i)}{p_i}\right)\left(\sum_{i\in I_j}\frac{q_i\phi_i}{p_i(1-\phi_i)}\right)\\
    \le&\frac{4}{\rho}\left(\sum_{i\in I_j}\frac{q_i\phi_i}{p_i}\right)\left(\sum_{i\in I_j}\frac{q_i\phi_i}{p_i}\right)\\
    \le&\frac{4n_j}{\rho}\sum_{i\in I_j}\left(\frac{q_i\phi_i}{p_i}\right)^2\\
    =&\frac{n_j}{\rho}\sum_{i\in I_j}\left(\frac{q_ig_i}{p_i}\right)^2.
    \addtocounter{equation}{1}\tag{\theequation}\label{eq:bound_first_1_0}
  \end{align*}
  We perform a weak learnability assumption here such that for some $\delta>0$ we have
  \begin{equation}
    \label{eq:bound_first_1_1}
    \left(\sum_{i\in I_j}\frac{q_ig_i}{p_i}\right)^2\ge4\delta^2\sum_{i\in I_j}\left(\frac{q_ig_i}{p_i}\right)^2.
  \end{equation}
  Combining (\ref{eq:bound_first_1_0}) and (\ref{eq:bound_first_1_1}), it yields
  $$\left(\sum_{i\in I_j}\frac{q_ig_i}{p_i}\right)^2\left(\sum_{i\in I_j}\frac{q_ih_i}{p_i}\right)^{-1}\ge\frac{4\delta^2\rho}{n_j}\left(\sum_{i\in I_j}\frac{q_ig_i^2}{p_ih_i}\right),$$
  which completes this proof.
\end{proof}

\begin{lemma}
  \label{lem:bound_first_2}
  $\bm g\tran\bm Q(\bm H\bm Q)^{-1}\bm Q\bm g\le\sum\frac{q_i}{p_i}\ell(y_i)$.
\end{lemma}

\begin{proof}
  \citep{Sun2014}'s theorem 4 proved $\frac{g_i^2}{h_i}\ge\ell(y_i)$, which gives
  $$\frac{\left(\frac{q_ig_i}{p_i}\right)^2}{\frac{q_ih_i}{p_i}}=\frac{q_ig_i^2}{p_ih_i}\ge\frac{q_i}{p_i}\ell(y_i),$$
  completing this proof.
\end{proof}

Finally, we are now ready to prove theorem~\ref{the:first1}.

\begin{proof}
  Applying Taylor's theorem on $L(\cdot)$ between $\bm y$ and $\bm y+\bm f$ yields
  \begin{align*}
    &L(\bm y+\bm f)\\
    =&L(\bm y)+\nabla L(\bm y)\tran\bm f+\frac{1}{2}\bm f\tran\nabla^2L(\bm\xi)\bm f\\
    =&L(\bm y)+\bm g\tran\bm Q\bm f+\frac{1}{2}\bm f\tran\bm H_{\bm\xi}\bm Q\bm f,
  \end{align*}
  where $\bm\xi$ is a vector between $\bm y$ and $\bm y+\bm f$. Replacing $\bm f$ with the given Newton step yields
  \begin{align*}
    L(\bm y+\bm f)=L(\bm y)-\nu\bm g\tran\bm Q\bm V(\bm V\tran\bm H\bm Q\bm V)^{-1}\bm V\tran\bm Q\bm g\\
    +\frac{\nu^2}{2}\bm g\tran\bm Q\bm V(\bm V\tran\bm H\bm Q\bm V)^{-1}(\bm V\tran\bm H_{\bm\xi}\bm Q\bm V)\\
    (\bm V\tran\bm H\bm Q\bm V)^{-1}\bm V\tran\bm Q\bm g.
  \end{align*}
  Applying lemma~\ref{lem:bound_first_0} yields
  \begin{align*}
    L(\bm y+\bm f)\le L(\bm y)\\
    +\left(\frac{\mu\nu^2}{2}-\nu\right)\bm g\tran\bm Q\bm V(\bm V\tran\bm H\bm Q\bm V)^{-1}\bm V\tran\bm Q\bm g.
  \end{align*}
  Applying lemma~\ref{lem:bound_first_1} yields
  $$L(\bm y+\bm f)\le L(\bm y)+\gamma^2\left(\frac{\mu\nu^2}{2}-\nu\right)\bm g\tran\bm Q(\bm H\bm Q)^{-1}\bm Q\bm g.$$
  Applying lemma~\ref{lem:bound_first_2} yields
  $$L(\bm y+\bm f)\le\left[1-\gamma^2\left(\nu-\frac{\mu\nu^2}{2}\right)\right]L(\bm y).$$
  By setting a proper shrinkage $\nu$ such that $0<\gamma^2\left(\nu-\frac{\mu\nu^2}{2}\right)<1$, a linear convergence rate can be established.
\end{proof}

\bibliography{main}

\end{document}